\documentclass[11pt]{article}

\usepackage{acl}

\usepackage{times}
\usepackage{latexsym}

\usepackage[T1]{fontenc}

\usepackage[utf8]{inputenc}

\usepackage{microtype}

\usepackage{inconsolata}

\usepackage{graphicx}

\usepackage[ruled, lined]{algorithm2e}

\usepackage{times}
\usepackage{latexsym}
\usepackage{amsmath}      
\usepackage{amsthm}       
\usepackage{amssymb}      

\usepackage{booktabs}
\usepackage{multirow}

\newtheorem{theorem}{Theorem}[section]
\newtheorem{lemma}[theorem]{Lemma} 
\usepackage{comment}

\title{SAGE: Sign-Adaptive Gradient for Memory-Efficient LLM Optimization}

\author{
  \textbf{Wooin Lee}\textsuperscript{1}\thanks{Corresponding author},
  \textbf{Hyuntae Kim}\textsuperscript{2}
\\\\
  \textsuperscript{1}London, United Kingdom,
  \textsuperscript{2}Seoul, Republic of Korea
\\
  \small{\{\textsuperscript{1}\texttt{naubull2}, \textsuperscript{2}\texttt{kimhntae}\}@gmail.com}
}

\begin{document}
\maketitle
\begin{abstract}
The AdamW optimizer, while standard for LLM pretraining, is a critical memory bottleneck, consuming optimizer states equivalent to twice the model's size. Although light-state optimizers like SinkGD attempt to address this issue, we identify the \textit{embedding layer dilemma}: these methods fail to handle the sparse, high-variance gradients inherent to embeddings, forcing a hybrid design that reverts to AdamW and partially negates the memory gains. We propose \textbf{SAGE} (Sign Adaptive GradiEnt), a novel optimizer that resolves this dilemma by replacing AdamW in this hybrid structure. \textbf{SAGE} combines a Lion-style update direction with a new, memory-efficient $O(d)$ adaptive scale. This scale acts as a "safe damper," provably bounded by $1.0$, which tames high-variance dimensions more effectively than existing methods. This superior stability allows \textbf{SAGE} to achieve better convergence. On Llama models up to 1.3B parameters, our \textbf{SAGE}-based hybrid achieves new state-of-the-art perplexity, outperforming all baselines, including SinkGD hybrid, while significantly reducing optimizer state memory.
\footnote{Implementation is available at: \url{https://github.com/naubull2/SAGE-optimizer}}
\end{abstract}

\section{Introduction}
\label{sec:intro}
The training of Large Language Models (LLMs) \citep{gpt3, llama} is a resource-intensive endeavor, fundamentally constrained by the "memory wall." While significant strides have been made to reduce the memory footprint of model parameters—through techniques such as Low-Rank Adaptation (LoRA) \citep{hu2022lora} or low-precision quantization \citep{dettmers2022llm}—these approaches often act as compromises, trading off model capacity or training dynamics for memory efficiency. However, a distinct and equally critical memory bottleneck persists in the optimizer itself. The industry-standard AdamW optimizer \citep{loshcilov19}, prized for its stability, demands two full-sized moment states for every parameter. For modern multi-billion parameter models, these $O(Vd)$ states can consume memory proportional to twice the model size, drastically limiting feasible batch sizes and model scaling.

To alleviate this, a lineage of light-state optimization research has emerged. Sign-based methods like Lion \citep{lion} demonstrated the viability of using a single moment state with sign-based updates. Methods like GaLore \citep{ga_lore} and APOLLO \citep{apollo} explored low-rank projections and block-wise updates to reduce state overhead. Most recently, SWAN \citep{swan} and SinkGD \citep{scetbon25} proposed a stateless $O(1)$ normalization technique, representing the latest advancement in minimizing memory overhead while maintaining training stability.

Building upon these advances, we observed an intriguing phenomenon regarding the embedding layer. While recent light-state methods are highly effective for dense layers, we found that the sparse, high-variance gradients characteristic of the embedding layer often pose a unique stability challenge. This observation explains the design choice in recent works to retain a hybrid structure, falling back to the memory-intensive AdamW specifically for embeddings. Our work takes this observation as a starting point: rather than viewing this as a limitation of prior work, we see it as an opportunity to develop a specialized, memory-efficient solution that eliminates this final dependency on heavy optimizers (Section \ref{sec:motivation}).
\\\\
\noindent\textbf{Contributions.} Our core contribution is to solve this \textit{embedding layer dilemma}. We introduce \textbf{SAGE} (\textbf{S}ign \textbf{A}daptive \textbf{G}radi\textbf{E}nt), a novel, memory-efficient optimizer specifically designed to master the embedding layer's gradients. \textbf{SAGE} maintains only a single $O(Vd)$ moment state (like Lion) but replaces AdamW's $O(Vd)$ second-moment state with a novel, $O(d)$ dimension-wise \textbf{adaptive damper}. This state tracks the mean absolute gradient ($L_1$ norm) and is theoretically bounded, providing a memory-efficient mechanism to stabilize high-variance gradients.
This design results in a truly memory-efficient hybrid optimizer. We demonstrate that our \textbf{SAGE}-based hybrid, which allows for a more aggressive learning rate, outperforms AdamW, Lion, and the original SinkGD+Adam hybrid in both test perplexity and memory footprint across models up to 1.3B parameters (Section \ref{sec:experiments}).

\section{Background and Motivation}
\label{sec:motivation}

\subsection{The Optimizer State Memory Bottleneck}
The memory cost of an optimizer is dominated by its states. For a parameter $\theta$, AdamW maintains a first moment $\mathbf{M}_t$ and a second moment $\mathbf{V}_t$. For an embedding layer of size $V \times d$, this requires two $O(Vd)$ states. As $V$ often exceeds $100,000$, these states are the single largest memory cost outside of the model parameters themselves.

\subsection{The Promise of Light-State Optimizers}
Our work is built upon the foundations laid by several key paradigms in optimization. We view the following methods as the cornerstones of our approach, synthesizing their strengths to achieve a new level of efficiency:
\begin{itemize}
    \item \textbf{AdamW} \citep{loshcilov19}: The robust standard, providing the stability of second-moment normalization which we aim to approximate efficiently.
    \item \textbf{Lion} \citep{lion}: The pioneer of the single-state, sign-based update, which demonstrates that memory can be halved without sacrificing convergence speed.
    \item \textbf{SinkGD} \citep{scetbon25}: The state-of-the-art framework for light-state training, which provides the effective hybrid architecture that we adopt and refine.
\end{itemize}

\subsection{The Embedding Layer Dilemma}
\label{ssec:dilemma}
While promising, we observe a curious escape hatch in the design of state-of-the-art light-state optimizers. The work by \citet{scetbon25} itself defaults to a hybrid structure, applying its stateless SinkGD to 2D weights but reverting to the 2-state AdamW for the embedding layer.

This raises a critical question: why does the light-state approach fail for embeddings? We hypothesize this is due to the unique properties of the embedding layer: its gradients are both \textbf{sparse} and \textbf{high-variance} due to the Zipfian distribution \citep{zipf1946psychology} of token frequencies. Stateless methods, which apply uniform normalization, struggle to efficiently learn representations under these conditions.

To validate this hypothesis, we conduct a motivating experiment comparing two versions of SinkGD on a 270M parameter model:
\begin{enumerate}
    \item \textbf{SinkGD-Pure:} The stateless normalizer is applied to \textbf{all} 2D parameters, including embeddings.
    \item \textbf{SinkGD-Hybrid:} The hybrid design from \citet{scetbon25}, using AdamW for the embedding layer.
\end{enumerate}

\begin{figure}[h]
  \centering
  \includegraphics[width=1.0\columnwidth]{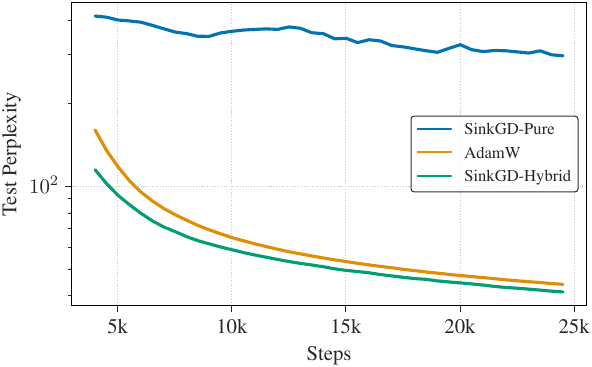}
  \caption{Test Perplexity for SinkGD-Pure vs. SinkGD-Hybrid. Even with tuned learning rates, the pure-stateless optimizer fails to learn effective embedding representations, resulting in significantly worse final perplexity compared to the hybrid design.}
  \label{fig:sinkgd_paradox}
\end{figure}

As shown in Figure \ref{fig:sinkgd_paradox}, while SinkGD-Pure can be tuned to remain stable, it suffers from a catastrophic deficit in learning efficiency, failing to reach competitive perplexity. This confirms our hypothesis: the embedding layer requires specialized, adaptive optimization that naive light-state methods cannot provide.

This result provides strong evidence for our central hypothesis: the embedding layer is a critical point of instability that naive light-state optimizers cannot handle. Therefore, the key to a truly memory-efficient optimizer is to design a new, light-state algorithm that can successfully replace AdamW on the embedding layer.

\section{Proposed Method: SAGE}
\label{sec:method}

\subsection{A Hybrid, Efficient Optimizer}
Our analysis in Section \ref{ssec:dilemma} establishes a clear design goal: an embedding optimizer that is both memory-light and stable against sparse, high-variance gradients.

The Lion optimizer \cite{lion} is a strong memory-light baseline, using only a single $O(Vd)$ momentum state $\mathbf{M}_t$. Its core update, $\hat{\mathbf{U}}_t = \text{sign}(\mathbf{M}_t)$, can be understood as having an implicit, static update scale of $1.0$ in every dimension. While simple and fast, this static scale lacks any mechanism to control high-variance dimensions, which limits its maximum stable learning rate and can be suboptimal.

We propose \textbf{SAGE} as a direct, adaptive generalization of Lion. SAGE replaces Lion's static $1.0$ scale with a novel, $O(d)$ adaptive scale, $\mathbf{H}_t$.
This adaptive scale is the core of our contribution. It is an element-wise relative damper derived from an $O(d)$ exponential moving average(EMA) of the mean absolute gradients ($L_1$ norm). As detailed in Algorithm \ref{alg:sage}, this scale compares the "loudness" of each dimension to the layer update's root mean square(RMS) average. It is provably bounded ($||\mathbf{H}_t||_\infty \le 1.0$), allowing it to selectively damp high-variance dimensions where the signal exceeds the layer average, while defaulting to Lion's $1.0$ scale for quieter dimensions.

Our full optimizer follows the hybrid structure motivated in Section \ref{sec:motivation}. While prior work \cite{scetbon25} relied on AdamW for 1D parameters, we find that SAGE is equally effective for them. We therefore propose a configuration summarized in Table \ref{tab:hybrid_structure}:

\quad
\begin{table}[h]
\centering
\small
\begin{tabular}{llc}
  \toprule
  \textbf{Model Component} & \textbf{Optimizer} & \textbf{State Size} \\
  \midrule
  Embedding Layer & \textbf{SAGE} & $O(Vd) + O(d)$ \\
  Biases/Norms (1D) & \textbf{SAGE} & $2 \times O(d)$ \\
  Dense Weights (2D) & SinkGD & $O(1)$ \\
  \bottomrule
\end{tabular}
\caption{Composition of the proposed \textbf{SAGE-Hybrid} optimizer. By using SAGE for embeddings instead of AdamW (which requires $2 \times O(Vd)$ states), we reduce the dominant memory cost by $\approx 50\%$ while maintaining stability.}
\label{tab:hybrid_structure}
\end{table}
\quad

\subsection{The SAGE Algorithm}
The SAGE update is detailed in Algorithm \ref{alg:sage}. While primarily designed for the embedding layer, SAGE generalizes naturally to 1D parameters (biases, layer norms) via simple branching logic (Equation \ref{eq:s_t_calc}), allowing it to act as a universal stateful optimizer.

\begin{algorithm}
  \DontPrintSemicolon 
  \caption{Sign Adaptive GradiEnt}
  \label{alg:sage}
  
  \textbf{Input:} Learning rate schedule $(\eta_t)_{t \ge 1}$, decay rates $w, \beta_1, \beta_2, \epsilon$.\;
  \textbf{Initialize:} Params $\theta_0$, $\mathbf{M}_0 \leftarrow \mathbf{0}$, $\mathbf{S}_0 \leftarrow \mathbf{0}$.\;
  \BlankLine
  
  \For{$t = 1$ \KwTo $T$}{
      $g_t \leftarrow \nabla_{\theta} f_t(\theta_{t-1})$\;
      $\theta_{t-1} \leftarrow \theta_{t-1} \cdot (1 - \eta_t w)$\;
      \BlankLine
      
      \eIf{$\theta$ is $2D$ (Embedding)}{
          $(\mathbf{s}_t)_j \leftarrow \frac{1}{V} \sum_{i=1}^V |g_{t,ij}|$ \text{for} $j=1...d$\;
      }{
          $\mathbf{s}_t \leftarrow |g_t|$ \hfill \{Element-wise Abs\}\;
      }
      \BlankLine
      
      $\mathbf{S}_t \leftarrow \beta_2 \mathbf{S}_{t-1} + (1 - \beta_2) \mathbf{s}_t$\;
      $\hat{\mathbf{S}}_t \leftarrow \mathbf{S}_t / (1 - \beta_2^t)$\;
      \BlankLine
      
      $\sigma_{\text{rms}} \leftarrow \sqrt{\frac{1}{d}\sum_{j=1}^d (\hat{\mathbf{S}}_t)_j^2}$\;
      \BlankLine
      
      $\gamma_{\text{rms}} \leftarrow \sqrt{\frac{1}{d}\sum_{j=1}^d (\mathbf{s}_t)_j^2}$\;
      \BlankLine
      
      \For{$j=1$ \KwTo $d$}{
          $(\mathbf{D}^{\text{ema}}_t)_j \leftarrow \sigma_{\text{rms}} / ((\hat{\mathbf{S}}_t)_j + \epsilon)$\;
          $(\mathbf{D}^{\text{inst}}_t)_j \leftarrow \gamma_{\text{rms}} / ((\mathbf{s}_t)_j + \epsilon)$\;
          $(\mathbf{H}_t)_j \leftarrow \min((\mathbf{D}^{\text{ema}}_t)_j, (\mathbf{D}^{\text{inst}}_t)_j, 1)$\;
      } 
      
      $\mathbf{C}_t \leftarrow \text{sign}(\beta_1 \mathbf{M}_{t-1} + (1 - \beta_1) g_t)$\;
      $\hat{\mathbf{U}}_t \leftarrow \mathbf{C}_t \odot \mathbf{H}_t$\;
      $\theta_t \leftarrow \theta_{t-1} - \eta_t \cdot \hat{\mathbf{U}}_t$\;
      $\mathbf{M}_t \leftarrow \beta_2 \mathbf{M}_{t-1} + (1 - \beta_2) g_t$\;
  }
  \Return $\theta_T$\;
\end{algorithm}

\subsubsection*{Decoupled Weight Decay}
We follow the AdamW \cite{loshcilov19} style of decoupled weight decay, which is applied directly to the parameters $\theta$ before any other operation:
\begin{equation}
    \theta_{t-1} \leftarrow \theta_{t-1} \cdot (1 - \eta_t w)
\end{equation}

\subsubsection*{Adaptive Scale Calculation}
The core of SAGE's memory-efficiency is its $O(d)$ adaptive state, $\mathbf{S}_t$.
For embedding layers (2D), we compute $\mathbf{s}_t$ by taking the mean of absolute gradients across the vocabulary dimension $i\in\{1,...,V\}$ for each embedding dimension $j\in\{1,...,d\}$:
\begin{equation}
\label{eq:s_t_calc}
   \mathbf{s}_t \leftarrow \begin{cases} \frac{1}{V} \sum_{i=1}^V |g_{t,ij}| & \text{if 2D (Embedding)} \\ |g_t| & \text{if 1D (Bias/Norm)} \end{cases}%
\end{equation}
For 1D parameters, we set $\mathbf{s}_t = |g_t|$, retaining per-element resolution.
This vector $\mathbf{s}_t$ is used to update the exponential moving average (EMA) $\mathbf{S}_t$. To counteract bias towards zero during the initial training steps, we apply a standard bias correction \cite{kingma15} to obtain the final state $\hat{\mathbf{S}}_t$

The adaptive scale $\mathbf{H}_t$ is calculated using a \textit{Relative RMS} strategy. We first calculate the Root Mean Square ($\sigma_{\text{rms}}$) of the state vector $\hat{\mathbf{S}}_t$ to establish a reference "loudness" for the layer's update. We then define the relative adaptive scale $\mathbf{H}_t$ as the ratio of this reference to the $j$-th dimension's value.

This design of $\mathbf{H}_t$, formalized in Algorithm\ref{alg:sage}, ensures scale invariance. For "quiet" dimensions (where $(\hat{\mathbf{S}}_t)_j < \sigma_{\text{rms}}$), the ratio exceeds $1$, and we clip to $1$, defaulting to Lion's behavior. For "loud" dimensions (where $(\hat{\mathbf{S}}_t)_j > \sigma_{\text{rms}}$), the ratio is $< 1$, safely damping the high-variance signal.

Finally, to prevent catastrophic instability from sudden spikes that the EMA may lag behind, we enforce an additional instantaneous stability constraint derived from the current batch statistics (analogous to element-wise Adaptive Gradient Clipping \cite{brock2021agc}); full details are provided in Appendix \ref{sec:appendix_algo_details}.

\subsubsection*{Direction and Final Update}
The final parameter update is the element-wise product of the direction $\mathbf{C}_t$ and our new adaptive scale $\mathbf{H}_t$, scaled by the learning rate $\eta_t$. The update direction $\mathbf{C}_t$ and momentum state $\mathbf{M}_t$ are updated identically to Lion \cite{lion}.
\begin{equation}
    \hat{\mathbf{U}}_t \leftarrow \mathbf{C}_t \odot \mathbf{H}_t, \quad \theta_t \leftarrow \theta_{t-1} - \eta_t \cdot \hat{\mathbf{U}}_t
\end{equation}

\subsection{Theoretical Analysis}
We now provide a theoretical analysis to ground our approach. The convergence of our hybrid optimizer relies on its components; we refer to the established proofs for AdamW \cite{loshcilov19} and Sinkhorn GD \cite{scetbon25}. Our analysis focuses on proving the convergence of our novel \textbf{SAGE} update.

Following standard practice in non-convex optimization \cite{nesterov04, robbins51sgd, bottou18}, we make the following assumptions:
\begin{itemize}
  \item \textbf{L-smoothness:} The loss function $f$ is $L$-smooth, ensuring the gradient does not change arbitrarily fast.
  \item \textbf{Bounded Variance:} The stochastic gradients $g_t$ have bounded variance,\\
           $\mathbb{E}[||g_t - \nabla f(\theta_t)||^2] \le \sigma^2$.
  \item \textbf{Bounded Gradients:} The stochastic gradients are bounded, $||g_t||_\infty \le G$.
\end{itemize}

\paragraph{Justification for the Adaptive State.}
Our use of the $L_1$-based state $\mathbf{S}_t$ as a proxy for per-dimension gradient scale is grounded in the established principles of deep network training. Modern architectures, through techniques like Xavier initialization \cite{glorot10} and Layer Normalization \cite{ba16layernorm}, are designed to maintain activation and gradient means near zero. Under this zero-mean assumption, the mean absolute value ($\mathbf{S}_t \approx \mathbb{E}[|g_t|]$) becomes a robust estimator of the gradient's standard deviation (scale), similar to the $\sqrt{\mathbb{E}[g_t^2]}$ estimator used by Adam.

\begin{lemma}[Bounded $\mathbf{S}_t$]
\label{lemma:bounded_s_t}
\textit{Given the Bounded Gradients assumption, the stochastic gradient $g_t$ is bounded. Therefore, the per-dimension L1-mean $\mathbf{s}_t$ (an average of bounded values) is also bounded. As $\mathbf{S}_t$ is an EMA of this bounded signal, $\mathbf{S}_t$ is also bounded, i.e., $|| \mathbf{S}_t ||_\infty \le S_{\text{max}}$.}
\end{lemma}

\begin{lemma}[Bounded Adaptive Scale $\mathbf{H}_t$]
\label{lemma:bounded_h_t}
\textit{The SAGE adaptive scale $\mathbf{H}_t$ is, by design, element-wise bounded, with an $L_\infty$ norm $||\mathbf{H}_t||_\infty \le 1.0$.}
\end{lemma}

\begin{proof}
Let $\hat{\mathbf{S}}_t$ be the bias-corrected EMA state. From Lemma \ref{lemma:bounded_s_t}, $\hat{\mathbf{S}}_t$ is bounded and non-negative.
We calculate the layer-wise RMS, $\sigma_{\text{rms}} = \sqrt{\frac{1}{d}\sum (\hat{\mathbf{S}}_t)_j^2}$, which serves as a strictly positive scalar reference (assuming non-zero gradients).
The raw damper $\mathbf{D}_t$ is calculated as $(\mathbf{D}_t)_j = \sigma_{\text{rms}} / ((\hat{\mathbf{S}}_t)_j + \epsilon)$.

The final adaptive scale $\mathbf{H}_t$ is computed element-wise as:
\begin{equation}
\label{eq:h_t_clamp}
    (\mathbf{H}_t)_j = \min((\mathbf{D}_t)_j, 1)
\end{equation}
By the definition of the $\min$ operation, for any dimension $j$, the value $(\mathbf{H}_t)_j$ is non-negative and guaranteed to be $(\mathbf{H}_t)_j \le 1.0$.
Therefore, the $L_\infty$ norm (the maximum absolute element) of the adaptive scale is also bounded:
\begin{equation}
\label{eq:h_norm_final}
    ||\mathbf{H}_t||_\infty = \max_j |(\mathbf{H}_t)_j| \le 1
\end{equation}
This formally proves that our adaptive scale $\mathbf{H}_t$ provides a provably safe, non-amplifying update.
\end{proof}

\begin{theorem}[Convergence of SAGE]
\label{theorem:sage_convergence}
\textit{SAGE converges to a stationary point (i.e., $\lim_{T \to \infty} \mathbb{E}[||\nabla f(\theta_t)||^2] = 0$) under the standard assumptions and a suitable learning rate decay schedule.}
\end{theorem}

\begin{proof}[Proof Sketch]
The \textbf{SAGE} update is $\hat{\mathbf{U}}_t = \mathbf{C}_t \odot \mathbf{H}_t$.

\begin{enumerate}
  \item The update direction $\mathbf{C}_t = \text{sign}(\dots)$ is bounded, $||\mathbf{C}_t||_\infty \le 1$.
  \item From Lemma \ref{lemma:bounded_h_t}, the adaptive scale $\mathbf{H}_t$ is also bounded, $||\mathbf{H}_t||_\infty \le 1$.
  \item Therefore, the full update step $\hat{\mathbf{U}}_t$ is element-wise bounded, $||\hat{\mathbf{U}}_t||_\infty = ||\mathbf{C}_t \odot \mathbf{H}_t||_\infty \le 1$.
\end{enumerate}

\textbf{SAGE} applies an update that is provably no larger in magnitude than that of Lion. As \textbf{SAGE} is a strictly safer (i.e., more cautious) sign-based method, its convergence is guaranteed by the established frameworks for $\text{sign}$-based optimizers \cite{lion}.
\end{proof}

\subsection{Relationship to Lion}
\label{ssec:relationship_lion}
\textbf{SAGE} is a direct, adaptive generalization of the Lion\cite{lion} optimizer. The core Lion update can be expressed as:
\begin{equation}
    \hat{\mathbf{U}}_t^{\text{Lion}} = \mathbf{C}_t \odot \mathbf{1}
\end{equation}
where $\mathbf{C}_t = \text{sign}(\beta_1 \mathbf{M}_{t-1} + (1 - \beta_1) g_t)$ and the update magnitude is a static scale of $\mathbf{1}$.  The \textbf{SAGE} update simply replaces this static scale with our dynamic, adaptive scale $\mathbf{H}_t$: \begin{equation} \hat{\mathbf{U}}_t^{\text{SAGE}} = \mathbf{C}_t \odot \mathbf{H}_t \end{equation} Therefore, Lion is a special case of \textbf{SAGE} where $\mathbf{H}_t$ is fixed to $\mathbf{1}$ for all dimensions and all steps $t$.

As our theoretical analysis in Lemma \ref{lemma:bounded_h_t} proves, $||\mathbf{H}_t||_\infty \le 1.0$. This means \textbf{SAGE} is a strictly safer generalization, as its per-dimension update magnitude is provably less than or equal to Lion's. Our experiments confirm that this bounded, adaptive damping is the key to providing the stability needed to outperform Lion by using a higher learning rate.

\begin{figure*}[ht]
\centering
\includegraphics[width=\textwidth]{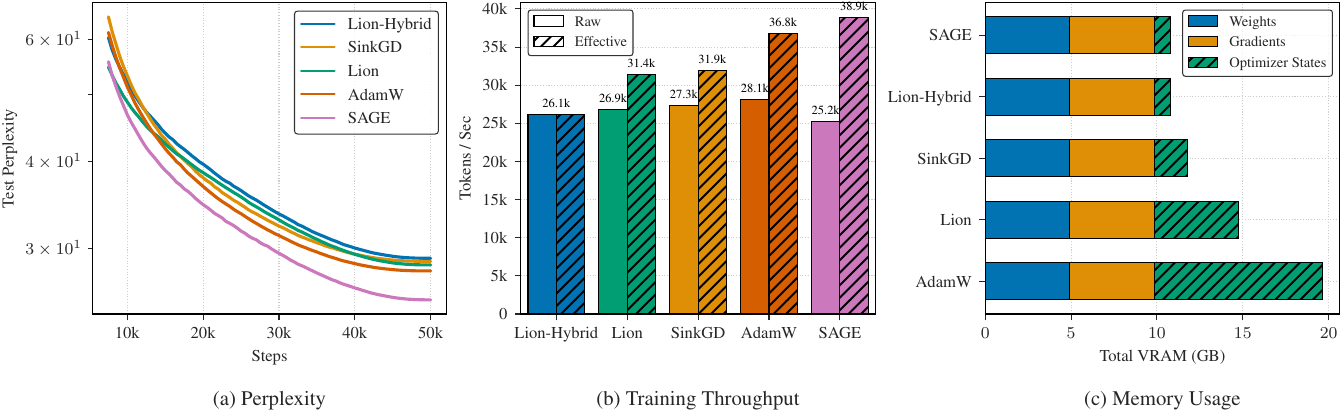} 
\caption{\textbf{Convergence and efficiency on the 1.3B model.} (a) Test perplexity - lower the better; SAGE achieves the lowest final perplexity and faster convergence. (b) Throughput - higher the better; \textit{Effective} throughput normalizes processing speed by the steps required to match the lowest (Lion-Hybrid) performance. (c) Peak memory usage breakdown by component.}
\label{fig:convergence_and_throughput}
\end{figure*}

\begin{table*}[ht]
\centering
\begin{tabular}{ll cc cc cc}
\toprule
\multirow{2}{*}{\textbf{Group}} & \multirow{2}{*}{\textbf{Method}} & \multicolumn{2}{c}{\textbf{270M}} & \multicolumn{2}{c}{\textbf{0.6B}} & \multicolumn{2}{c}{\textbf{1.3B}} \\
\cmidrule(lr){3-4} \cmidrule(lr){5-6} \cmidrule(lr){7-8}
& & \textbf{PPL} & \textbf{Mem (GB)} & \textbf{PPL} & \textbf{Mem (GB)} & \textbf{PPL} & \textbf{Mem (GB)} \\
\midrule
\multirow{4}{*}{\textit{Baselines}} & AdamW & 37.35 & 2.1 & 31.74 & 4.4 & 27.81 & 9.8 \\
& APOLLO & 52.83 & 1.2 & 44.84 & 2.1 & 39.54 & 3.6 \\
& Lion & 30.24 & 1.0 & \textbf{26.58} & 2.2 & 28.37 & 4.9 \\
& SinkGD-Hybrid & 34.30 & 0.9 & 32.85 & 1.5 & 28.71 & 1.9 \\
\midrule
\multirow{3}{*}{\textit{Ablations}} & SinkGD-Pure & 192.7 & 0.0 & 117.8 & 0.0 & 98.84 & 0.0 \\
& SAGE-Pure & 116.0 & 1.0 & 179.0 & 2.2 & 216.0 & 4.9 \\
& Lion-Hybrid & 32.10 & \textbf{0.5} & 28.73 & \textbf{0.7} & 28.40 & \textbf{0.9} \\
\midrule
\multirow{1}{*}{\textit{Our Methods}}
& \textbf{SAGE-Hybrid} & \textbf{29.95} & \textbf{0.5} & 26.71 & \textbf{0.7} & \textbf{24.33} & \textbf{0.9} \\
\bottomrule
\end{tabular}
\caption{Main results comparing final test perplexity (PPL) and peak optimizer memory usage (Mem) across three model scales. SAGE consistently achieves the best perplexity while maintaining a low memory footprint, outperforming both pure and hybrid baselines. Lower PPL is better.}
\label{tab:main_results}
\end{table*}

\section{Experiments}
\label{sec:experiments}

We conduct a series of experiments to validate SAGE and answer three key questions:
1) Does our hybrid SAGE optimizer outperform state-of-the-art baselines in final perplexity and convergence speed?
2) Is SAGE's adaptive "safe-damping" mechanism the key to its stability and performance?
3) Does SAGE deliver on its promise of significantly reducing memory consumption compared to AdamW used in SinkGD?

\subsection{Experimental Setup}
\label{ssec:setup}

\paragraph{Training Details}
All models were implemented in PyTorch \cite{pytorch19} using the Transformers library \cite{huggingface19} and trained from scratch on a single NVIDIA H200 GPU. We used a global batch size of 130k tokens (sequence length 512) for all runs. All runs used a cosine learning rate schedule with a 10\% warmup. For baseline optimizers, we primarily adopted the learning rates \textit{favored} by their original works. We performed a sanity check on these baselines using larger learning rates, confirming that the hyperparameters suggested in their respective papers were optimal for stable convergence. Crucially, for our strongest direct competitor, Lion-Hybrid, we applied the same learning rates used for SAGE. For our proposed SAGE and other hybrid configurations not covered in prior work, we conducted a grid search to find the reasonable learning rate. This approach ensures a fair comparison where each optimizer operates in a reasonably tuned regime, while acknowledging that further fine-grained tuning could potentially yield marginal improvements for all methods. Other hyperparameters were kept constant ($\beta_1=0.9, \beta_2=0.99, \text{wd}=0.01$ for all optimizers).

\paragraph{Models and Data}
We pretrained a series of LLaMA-architecture \cite{llama} models from scratch, with parameter counts of 270M, 0.6B, and 1.3B. We trained all models on a 6.6B token subset of The Pile \cite{pile}, using \textbf{bfloat16} precision for efficiency. We chose The Pile over simpler datasets like C4 \cite{c4dataset} as recent work suggests its greater diversity provides a more rigorous benchmark for evaluating model capabilities \cite{bandari24}. As for the subset of The Pile, we used the data selection framework based on importance resampling introduced by \citet{dsir23}. We used an 8:2 train/test split ratio.

\paragraph{Optimizers and Baselines}
We compare the following list of optimizer configurations. Our primary method is the \textbf{SAGE}, which uses SAGE for the embedding layer and follows SinkGD for all other weights. For brevity, we hereinafter refer to the hybrid configurations simply as SAGE and SinkGD, unless the distinction with the pure ablation is necessary. We compare against:
\begin{itemize}
    \item \textbf{AdamW}: The standard 2-state (M, V) baseline from \citet{loshcilov19}.
    \item \textbf{Lion}: The sign-based baseline from \citet{lion} with 1-state (M) applied to all parameters.
    \item \textbf{APOLLO}: The baseline from \citet{apollo}, which uses an auxiliary low-rank optimizer state based on pure random projection.
    \item \textbf{SinkGD-Pure}: An ablation applying SinkGD from \citet{scetbon25} to all weights, including the embeddings.
    \item \textbf{SinkGD}: The baseline from \citet{scetbon25}, where SinkGD is applied to all 2D weights except for the embeddings and for all other weights AdamW is used. 
    \item \textbf{Lion-Hybrid}: A strong baseline with the same setup as SinkGD, but using Lion for embeddings.
    \item \textbf{SAGE-Pure}: An ablation where SAGE is applied to all parameters, including embeddings and 1D weights, to validate our hybrid design.
    \item \textbf{SAGE-Hybrid}: Our proposed method, using SinkGD for all 2D dense weights except for the embeddings and SAGE is applied to all other parameters. We also evaluated a variant using AdamW for 1D parameters as did \citet{scetbon25}, but found no practical difference in performance, so we report results for the simpler hybrid configuration.
\end{itemize}
To validate our design, we conducted ablation studies with \textbf{pure} and other \textbf{hybrid} configurations, confirming that the SAGE structure is the optimal combination. All 1D parameters (biases, normalization) of hybrid optimizers use AdamW as did in \citet{scetbon25}.
All experiments were run with 3 different random seeds, and we report the average of the final test perplexity.
Further details on data preprocessing and hyperparameter tuning are provided in Appendix \ref{sec:appendix}.

\paragraph{Unit-Norm SinkGD for Dense Layers.}
We employed a modified \textbf{Unit-Norm} version of Sinkhorn Gradient Descent (SinkGD) for dense 2D layers. Let $W \in \mathbb{R}^{m \times n}$ denote a weight matrix where we assume $m \le n$ without loss of generality. The original SinkGD \cite{scetbon25} explicitly scales the normalized gradient matrix by $\sqrt{nm}$ to match the expected Frobenius norm of Adam updates ($\approx \sqrt{nm}$).

However, our SAGE optimizer employs a damping mechanism that reduces the effective update magnitude for stability. Pairing SAGE with the original variance-preserving SinkGD creates a massive magnitude mismatch: dense layers receive updates scaled by $\sqrt{nm}$ while the embedding layer receives damped updates $\le 1.0$.

To resolve this, we removed the $\sqrt{nm}$ scaling factor from SinkGD, targeting a unit row norm ($1.0$). This aligns the update magnitude of the dense layers with the "safe-damped" regime of SAGE. We found that this magnitude alignment is critical for stability at scale, allowing us to control the effective learning speed via a single scalar hyperparameter $\alpha$ (typically $\alpha \approx 10$) rather than implicit dimension-dependent scaling.

\subsection{Main Results}
\label{ssec:main_results}

\paragraph{Performance at Scale} As shown in Table \ref{tab:main_results}, our SAGE optimizer achieves the lowest test perplexity across all 3 model scales. We observe that while the Lion-Hybrid is a strong baseline, our SAGE consistently outperforms it. This confirms our hypothesis that Lion's static $1.0$ scale is suboptimal, and our adaptive, "safe-damping" scale provides a significant performance gain.  Furthermore, Figure \ref{fig:convergence_and_throughput} visualizes these results, showing that the performance gap between SAGE and its baselines widens as the model size increases. This suggests that SAGE's superior stability and adaptivity are even more critical for larger, more complex models.
The performance advantage of SAGE becomes increasingly pronounced at larger scales. For the 1.3B model, it achieves a perplexity of 24.33, surpassing the next-best baseline (AdamW at 27.81) by a substantial margin. This widening gap is particularly noteworthy because as models grow, the embedding layer's relative parameter share decreases (from 48\% to 20\% in our experiments). We hypothesize this trend arises because larger models are deeper and have higher-dimensional embeddings, making the gradient dynamics more complex. In this regime, SAGE's fine-grained, adaptive control over embedding updates becomes even more critical for overall model performance.

\begin{figure*}[ht]
\centering
\includegraphics[width=\textwidth]{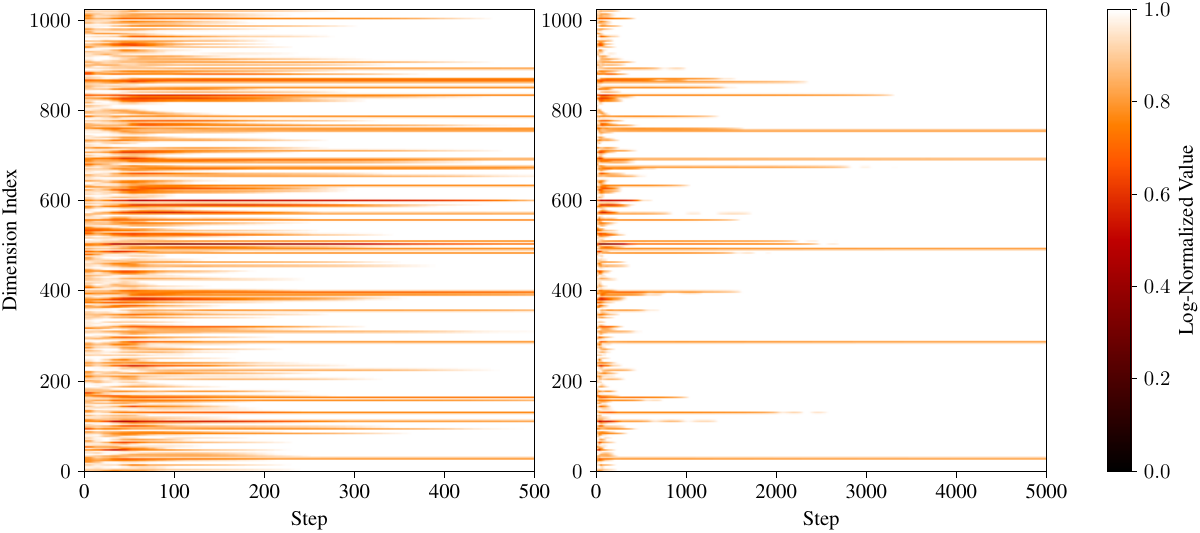}
\caption{Visualization of the SAGE adaptive scale $\mathbf{H}_t$ (for $d=1024$ dimensions) over the first 500 (left) and 5k (right) training steps. The raw $\mathbf{H}_t \in (0, 1.0]$ values are log-transformed and normalized. The \textbf{white background} (value 1.0) corresponds to the majority of "quiet" dimensions where the gradient variance is below the layer average, causing SAGE to clip the scale to 1.0 (Lion's default). The \textbf{colored stripes} (low values $\ll 1.0$) correspond to specific "loud" high-variance dimensions where SAGE actively applies strong damping. This confirms SAGE acts as a sparse damper, intervening only on specific unstable dimensions while letting the majority update at full scale.}
\label{fig:ht_heatmap}
\end{figure*}

\subsection{Memory Footprint}
\label{ssec:memory}
As shown in Table \ref{tab:main_results}, our SAGE optimizer offers significant memory savings compared to the SinkGD baseline, which uses the memory-intensive AdamW for its embedding layer. By replacing AdamW's $O(Vd)$ second-moment state with SAGE's negligible $O(d)$ state, our hybrid matches the low memory footprint of Lion-Hybrid. This reduction in optimizer state translates to a significant 10-45\% decrease in total peak memory, at the model size of 1.3B, enabling researchers to train larger models or use larger batch sizes on the same hardware. A detailed breakdown across all model scales can be found in Appendix \ref{sec:appendix_memory}.

\subsection{Analysis and Ablation Studies}
\label{ssec:analysis}

\paragraph{Learning Rate Tuning and Stability}
We investigated the optimal learning rate for SAGE through a targeted grid search, motivated by its relationship to Lion. \citet{lion} recommends a learning rate of $1 \times 10^{-4}$ for language modeling. This relationship (discussed in \ref{ssec:relationship_lion}) allows us to directly adapt the learning rates from Lion via a \textit{lazy-tuning} approach (tuned without extensive search), for SAGE. This can be thought of as a variation of LR grafting \citep{grafting22}. We therefore conducted a grid search over the range $[1 \times 10^{-4}, 2 \times 10^{-4}, 5 \times 10^{-4}, 1 \times 10^{-3}, 2 \times 10^{-3}]$. While Lion-Hybrid's performance peaks at $1 \times 10^{-4}$ and collapses shortly after, SAGE remains stable and achieves its best performance at a much higher learning rate of $1 \times 10^{-3}$. This demonstrates that SAGE's "safe-damping" mechanism successfully tames gradient variance, unlocking a more potent learning rate. Additionally, SAGE's adaptive nature also renders it robust to batch size variations, unlike stateless methods which behave like SGD and often require strict adherence to linear scaling rules when gradient noise levels change \citep{goyal17, trick14}.

\paragraph{Internal State Visualization}
Finally, in Figure \ref{fig:ht_heatmap}, we provide direct evidence of SAGE's mechanism. We visualize the normalized adaptive scale $\mathbf{H}_t$ over training time. The visualization reveals a distinct "sparse damping" policy. The vast majority of the heatmap is white (value $1.0$), identifying the low-variance dimensions where SAGE clips the scale to $1.0$, defaulting to Lion's behavior. Conversely, distinct colored stripes appear horizontally, \textbf{representing specific high-variance dimensions where SAGE actively applies strong damping} ($\mathbf{H}_t \ll 1.0$).
Crucially, the heatmap captures the temporal evolution of this mechanism. In the early training steps, Figure \ref{fig:ht_heatmap}(left), we observe dense, widespread damping activity as the optimizer calibrates to the initial gradient noise. Over the long term, Figure \ref{fig:ht_heatmap}(right), the system stabilizes: the damping becomes more sparse and localized to a consistent subset of structurally high-variance dimensions, allowing the majority of the model to update at full magnitude. This proves SAGE is not just reacting to random spikes, but applying a consistent, fine-grained policy that evolves from global stabilization to surgical intervention.

This low-dimensional structure is consistent with findings that neural network optimization occurs in a low-dimensional subspace \citep{gur2018gradient}, and that adaptation signals often have a low-rank structure \citep{aghajanyan2020intrinsic}. Our analysis shows that although $\mathbf{H}_t$ lives in a $d$-dimensional space, its evolution is governed by a few smooth, coherent principal directions. This indicates that SAGE maintains a compact representation of the update geometry, enabling stable optimization without relying on traditional per-parameter moment estimates. Further PCA analysis of $\mathbf{H}_t$ can be found in \ref{sec:appendix_pca}.

\section{Conclusion}
\label{sec:conclusion}
In this work, we identified the \textit{embedding layer dilemma} as an unaddressed bottleneck for light-state optimizers. We proposed SAGE, a novel optimizer that generalizes Lion with a memory-efficient $O(d)$ adaptive scale. By designing this scale as a "safe damper", SAGE effectively tames the high-variance embedding gradients that can limit the performance of the optimizer.
Our experiments on models up to 1.3B parameters show that SAGE achieves state-of-the-art perplexity and convergence speed. We demonstrated SAGE's superior stability allows for more aggressive learning rates than Lion, achieving top-tier performance while matching the low memory footprint of Lion-based hybrids. By resolving the \textit{embedding layer dilemma}, SAGE delivers optimal performance without the high memory overhead of traditional methods, offering a more powerful and accessible path for research involving LLM training.

\section*{Limitations}
While SAGE demonstrates strong performance, our study has several limitations. First, our experiments were conducted on models up to 1.3B parameters. Although we observe a clear trend of a widening performance gap as model size increases, these models are still relatively small in the modern LLM landscape, and further validation on larger scales (e.g., 7B+) is necessary. 

Second, the duration of our pre-training experiments was constrained by computational budgets. Since the \textit{Effective Throughput} metric is a function of the final perplexity reached by the baseline, the reported efficiency ratios are tied to this specific training horizon. In practice, performance gaps between optimization algorithms often widen as training extends on larger datasets; thus, our reported efficiency gains may represent conservative estimates compared to full-scale, multi-trillion token pre-training.

Additionally, we acknowledge the relevance of recent orthogonal optimization approaches such as Muon \cite{muon}. We selected SinkGD as our primary baseline because it represents the state-of-the-art within the strictly stateless and variance-preserving paradigm, aligning most closely with the theoretical foundations of our method. Extending our evaluation to include orthogonal families like Muon would require exhaustive hyperparameter tuning that exceeded our computational budget; however, we recognize this as a valuable direction for future large-scale benchmarking.

Finally, our analysis was confined to language modeling on The Pile dataset. The effectiveness of SAGE on other modalities (e.g., vision) or fine-tuning tasks remains an open question.

\bibliography{custom}

\clearpage
\onecolumn
\appendix
\section{Appendix}
\label{sec:appendix}
\subsection{Model Configurations}
Table \ref{tab:model_configs} provides the full configuration of each model used in the experiments, for readers to reproduce the results.

\begin{table}[h]
\centering
\begin{tabular}{lccc}
\toprule
\textbf{Configuration} & \textbf{270M} & \textbf{0.6B} & \textbf{1.3B} \\
\midrule
Hidden Size & 1024 & 1536 & 2048 \\
Intermediate Size & 2816 & 4224 & 5632 \\
Max Position Embeddings & 8192 & 8192 & 8192 \\
Num Attention Heads & 16 & 24 & 32 \\
Num Hidden Layers & 13 & 16 & 24 \\
Num Key-Value Heads & 2 & 3 & 4 \\
Vocab Size & 128256 & 128256 & 128256 \\
\bottomrule
\end{tabular}
\caption{Model configurations for the different scales used in the experiments.}
\label{tab:model_configs}
\end{table}

\subsection{Full Memory Footprint Comparison}
\label{sec:appendix_memory}
Table \ref{tab:memory_full} provides a comprehensive breakdown of memory consumption across all tested model scales.

\begin{table*}[h]
\centering
\begin{tabular}{l cc cc cc}
\toprule
& \multicolumn{2}{c}{\textbf{270M}} & \multicolumn{2}{c}{\textbf{0.6B}} & \multicolumn{2}{c}{\textbf{1.3B}} \\
& \multicolumn{2}{c}{\small(Emb. Params: 48\%)} & \multicolumn{2}{c}{\small(Emb. Params: 33\%)} & \multicolumn{2}{c}{\small(Emb. Params: 20\%)} \\
\cmidrule(lr){2-3} \cmidrule(lr){4-5} \cmidrule(lr){6-7}
\textbf{Optimizer} & \textbf{Opt. (GB)} & \textbf{Total (GB)} & \textbf{Opt. (GB)} & \textbf{Total (GB)} & \textbf{Opt. (GB)} & \textbf{Total (GB)} \\
\midrule
AdamW & 2.045 & 4.091 & 4.421 & 8.843 & 9.833 & 19.67 \\
APOLLO & 1.208 & 3.254 & 2.102 & 6.524  & 3.648 & 13.48 \\
Lion & 1.023 & 3.067 & 2.211 & 6.631 & 4.916 & 14.75 \\
SinkGD-Hybrid & 0.979 & 3.023 & 1.468 & 5.890 & 1.958 & 11.79 \\
\midrule
\textit{SinkGD-Pure} & \textit{0.0} & \textit{2.046} & \textit{0.0} & \textit{4.422} & \textit{0.0} & \textit{9.832}\\
SAGE-Pure & 1.023 & 3.069 & 2.212 & 6.634 & 4.92 & 14.75 \\
Lion-Hybrid & \textbf{0.489} & \textbf{2.535} & \textbf{0.734} & \textbf{5.156} & \textbf{0.979} & \textbf{10.81} \\
\midrule
\textbf{SAGE-Hybrid} & \textbf{0.489} & \textbf{2.535} & \textbf{0.734} & \textbf{5.156} & \textbf{0.979} & \textbf{10.81} \\
\bottomrule
\end{tabular}
\caption{Peak memory footprint comparison across all model scales. "Opt." refers to optimizer state memory, while "Total" is the sum of model weights, gradients and the optimizer states. SAGE is comparable to the low total memory of a pure stateless optimizer \textit{SinkGD-Pure}, offering a substantial reduction compared to the AdamW-based hybrid, which translates to significant savings in total memory.}
\label{tab:memory_full}
\end{table*}

\begin{figure*}[h]
    \centering
    \includegraphics[width=\textwidth]{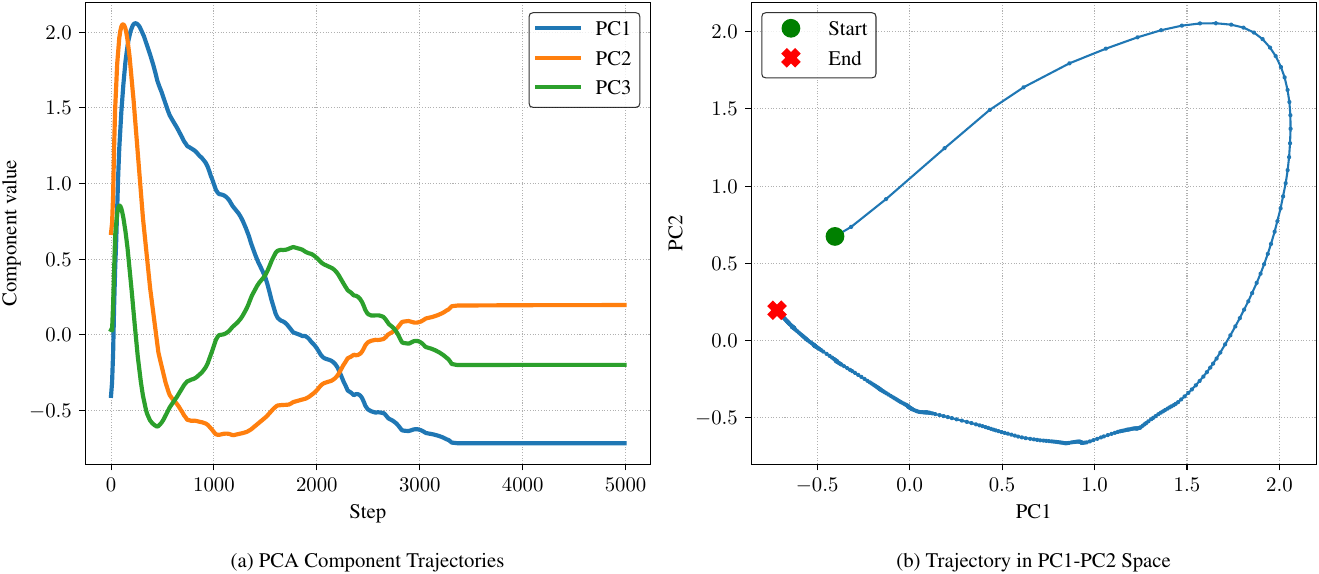}
    \caption{PCA of the adaptive scale vector $\mathbf{H}_t$. (a) Trajectories of the top 3 principal components. PC1 and PC2 show a sharp, distinct spike in activity during the early training phase (initial 1k steps), reflecting rapid adaptation to the embedding layer's geometry. This is followed by a gradual stabilization phase where components settle into steady states. (b) The trajectory in PC1-PC2 space reveals a clear "hook" pattern: a rapid initial excursion away from the origin (adaptation) followed by a smooth return path that converges toward a stable operating point (stabilization).}
    \label{fig:ht_pca_analysis}
\end{figure*}

\subsection{Internal State Dynamics Analysis}
\label{sec:appendix_pca}
\paragraph{Internal State Dynamics (PCA)}
To further understand how SAGE regulates update magnitudes, we performed Principal Component Analysis (PCA) on the sequence of the adaptive scale vectors, $\mathbf{H}_t$. The trajectories of the top principal components (PCs) reveal a distinct two-phase adaptation process.

As shown in Figure \ref{fig:ht_pca_analysis} (a), both PC1 and PC2 exhibit a sharp, significant rise in magnitude during the initial training steps ($t < 50$), followed by a gradual decay and stabilization. This indicates that SAGE performs \textbf{rapid initial calibration}, aggressively adjusting the damping factors as it first encounters the embedding layer's high-variance gradients. Following this discovery phase, the components settle into stable, non-zero values, confirming that the optimizer finds and maintains a consistent damping policy for specific dimensions rather than oscillating erratically.

The trajectory in the PC1-PC2 plane (Figure \ref{fig:ht_pca_analysis} (b)) visualizes this dynamic as a clear "excursion and return" path. The system quickly moves to a high-activity region to handle initial training instability and then smoothly converges toward a fixed point. This confirms that SAGE's adaptive mechanism is both responsive to early-stage noise and stable during long-term training.

\subsection{Learning Rate Tuning}
\label{sec:appendix_tuning}
We performed a comprehensive comparison runs of learning rates for SAGE on the 0.6B model, which resulted in $2\times10^{-3}$ as the best out of the range.

Furthermore, we observed a distinct advantage in stability regarding batch size variations. Stateless optimizers like SinkGD, lacking a history-based moment to normalize variance, behave similarly to standard SGD. Consequently, they are highly sensitive to the noise introduced by smaller batch sizes. To maintain convergence stability when the batch size is reduced, they typically require a proportional reduction in learning rate, following the Linear Scaling Rule \citep{goyal17, trick14}. In contrast, SAGE's adaptive damper $\mathbf{H}_t$ naturally adjusts to the changing gradient variance associated with different batch sizes. This behavior aligns with the properties of adaptive optimizers, which often follow a gentler square-root scaling or remain robust across a wider range of settings \citep{threemechanisms19}. Empirically, we found that SAGE could maintain stability using the same learning rate even when the effective batch size was reduced to fewer than 10, whereas SinkGD required significant re-tuning.

\subsection{Instantaneous Stability Constraint}
\label{sec:appendix_algo_details}
While the EMA-based state $\mathbf{S}_t$ effectively handles long-term variance, it introduces a lag that can be vulnerable to sudden, extreme gradient spikes. To address this, we implement an \textit{Instant Damper} based on the current batch's gradient $\mathbf{g}_t$.

We calculate the instantaneous relative scale $\mathbf{H}^{\text{inst}}_t$ using the same \textit{Relative RMS} logic but applied to $|\mathbf{g}_t|$ instead of $\mathbf{S}_t$:
\begin{equation}
    (\mathbf{H}^{\text{inst}}_t)_j = \frac{\sqrt{\text{mean}(|\mathbf{g}_t|^2)}}{|(\mathbf{g}_t)_j| + \epsilon}
\end{equation}
The final adaptive scale applied to the update is the conservative minimum of the EMA-based scale and this instantaneous scale:
\begin{equation}
    \mathbf{H}_t \leftarrow \min(\mathbf{H}^{\text{EMA}}_t, \mathbf{H}^{\text{inst}}_t, 1.0)
\end{equation}
This acts as a fast-acting \textit{circuit breaker} that clips updates immediately if the current gradient direction is an outlier relative to the layer's current activation pattern.

\subsection{Throughput Analysis and Metrics}
To comprehensively evaluate training efficiency, we report performance on a single H200 GPU using two distinct metrics in Figure\ref{fig:convergence_and_throughput}(b):

\textbf{Raw Throughput.} This measures the standard computational speed of the optimizer, expressed in tokens processed per second (tokens/sec). It reflects the computational overhead of the optimizer's update step (e.g., SAGE vs. AdamW) independent of convergence quality.

\textbf{Effective Throughput.} Raw throughput does not account for algorithmic efficiency (i.e., how much the model learns per step). Following the methodology of \citet{scetbon25}, we compute \textit{Effective Throughput} to quantify wall-clock convergence speed.

We define the effective throughput $\mathcal{T}_{\text{eff}}$ of an optimizer $\mathcal{O}$ as:
\begin{equation}
\mathcal{T}_{\text{eff}}(\mathcal{O}) = \frac{N_{\text{base}}}{T_{\mathcal{O}}}
\end{equation}
where $N_{\text{base}}$ is the total number of training tokens required by the baseline optimizer Lion-Hybrid, which exhibited the slowest convergence in our 1.3B scale experiments, to reach its final test perplexity, and $T_{\mathcal{O}}$ is the wall-clock time in seconds required by optimizer $\mathcal{O}$ to reach that same perplexity. This metric effectively penalizes optimizers that process tokens quickly but converge slowly, while rewarding optimizers like SAGE that maximize learning efficiency per second.

\vspace{0.5em}
\noindent\textbf{Sensitivity to Experimental Setup.}
It is important to note that the absolute magnitudes of these throughput ratios may vary across different experimental environments. Raw throughput is sensitive to implementation details (e.g., kernel optimizations), while effective throughput depends heavily on the dataset size and total training budget. Specifically, the target perplexity defined by the baseline changes as training extends; typically, superior optimizers widen the performance gap on larger datasets or longer training horizons. Therefore, while computational constraints limited the scale of this study, we hypothesize that the reported efficiency gains are conservative estimates. We expect the relative performance ordering—with SAGE outperforming the baselines—to remain consistent or become more pronounced in larger-scale training regimes.

\section{Use of AI Assistants}
We acknowledge the use of AI coding assistants to refine the Python scripts used for generating the plots in this paper. We emphasize that all data points presented in the figures were loaded directly from raw experimental logs, and the AI tools were strictly limited to formatting and visualization tasks.

\end{document}